\PassOptionsToPackage{numbers, compress}{natbib}
\documentclass{article}

\usepackage[preprint]{ewrl_2026}
\usepackage{natbib}

\usepackage{amsmath, amssymb, amsthm}
\usepackage{algorithm}
\usepackage{algcompatible}
\usepackage{graphicx}
\usepackage{tabularx}
\usepackage{subcaption}
\usepackage{xcolor}         

\definecolor{rljgreen}{RGB}{0, 101, 0}
\newcommand{\algoname}{\text{CESA-LinUCB}}
\newcommand{\ESA}{\textcolor{rljgreen}{\textbf{\texttt{\algoname}}}}

\theoremstyle{plain}
\newtheorem{theorem}{Theorem}
\newtheorem{lemma}[theorem]{Lemma}
\newtheorem{proposition}[theorem]{Proposition}
\newtheorem{corollary}[theorem]{Corollary}

\theoremstyle{definition}
\newtheorem{definition}{Definition}
\newtheorem{assumption}{Assumption}

\usepackage[utf8]{inputenc} 
\usepackage[T1]{fontenc}    
\usepackage{hyperref}       
\usepackage{url}            
\usepackage{booktabs}       
\usepackage{amsfonts}       
\usepackage{nicefrac}       
\usepackage{microtype}      

\title{Learning When to Trust in Contextual Social Bandits}

\author{%
  Majid Ghasemi, and Mark Crowley \\
  Department of Electrical \& Computer Engineering\\
  University of Waterloo, Canada\\
  \texttt{\{majid.ghasemi, mark.crowley\}}@uwaterloo.ca \\
}

\begin{document}

\maketitle

\begin{abstract}
Robust reinforcement learning typically assumes that feedback sources are either globally trustworthy or corrupted within a fixed global budget. We identify a more subtle failure mode that escapes this dichotomy, which we call \emph{Contextual Sycophancy}. In this failure, evaluators are truthful in benign contexts but systematically biased in critical ones, so that no single evaluator is reliable everywhere and the corrupt evaluators may form a \emph{majority} in the contexts that matter. Our first result is an information-theoretic lower bound. We exhibit two problem instances that induce \emph{identical} social-feedback distributions yet have disjoint optimal actions, proving that \emph{any} algorithm relying on social feedback alone (including any robust aggregator, regardless of breakdown point) incurs $\Omega(T)$ latent regret. This shows that breaking contextual sycophancy is impossible without having some information. We then show that a sparse stream of ground-truth audits, available with probability $p_{\mathrm{aud}}$, is sufficient. We propose \ESA, which learns a per-evaluator contextual \emph{trust boundary} from audits and re-weights feedback accordingly, and we prove a high-probability latent-regret bound of $\tilde{\mathcal{O}}\!\big(\sqrt{T\,d_{VC}/p_{\mathrm{aud}}} + d\sqrt{T} + \epsilon_{\mathrm{tol}}T\big)$, where $d_{VC}$ is the complexity of the adversary's bias strategy. The audit-dependence $1/\sqrt{p_{\mathrm{aud}}}$ matches the information-theoretic necessity of audits. Empirically, \ESA\ recovers the ground truth when $80\%$ of the social layer is adversarial, a regime in which median- and mean-based robust baselines fail.
\end{abstract}

\section{Introduction}
\label{sec:intro}

The alignment of AI systems has largely transitioned from a problem of \emph{specification} (defining explicit reward functions) to one of \emph{adjudication}, in which agents learn from aggregated human feedback \citep{lindstrom2024ai}. This shift was driven by the difficulty of specifying complex goals without inducing reward hacking \citep{amodei2016concrete}, and Reinforcement Learning from Human Feedback (RLHF) \citep{christiano2017deep, ouyang2022training} has become the de facto standard for steering large language models \citep{kaufmann2024survey}.

The paradigm though rests on a fragile assumption: that the feedback layer is, on average, reliable. Empirical work challenges this, documenting annotator disagreement, dependence on subjective preference, and inconsistent thresholds for ``helpfulness'' and ``harmfulness'' \citep{yeh2024reliable}. More pointedly, evaluators exhibit \emph{sycophancy}, which is a tendency to agree with the user's stated beliefs rather than report objective truth \citep{sharma2023sycophancy, perez2023discovering, ghasemi2026objective}, which corrupts the feedback distribution in a structured, belief-correlated way. When evaluators are systematically biased, standard aggregation fails to recover the latent ground truth \citep{lazier2023fairness}. Robust RL has developed tools for adversarial corruption \citep{lykouris2018stochastic, bogunovic2021stochastic}, but these adopt a \emph{global} corruption model. In this model, evaluators are universally good, universally bad, or corrupt within a fixed global budget (call it $C$).

We argue that real-world bias is \emph{contextual}, and that the global model is therefore the wrong abstraction. An evaluator may answer objective questions (e.g., arithmetic, context A) honestly while strategically agreeing with a user's prior on contested topics (e.g., politics, context B). We call this the ``Jekyll \& Hyde''\footnote{\url{https://en.wikipedia.org/wiki/Strange_Case_of_Dr_Jekyll_and_Mr_Hyde}} problem. Under such contextual corruption, an agent that estimates a single global trust score per source may suffer what we term \emph{Contextual Objective Decoupling (COD)}. In COD, agent's learned policy decouples from the ground truth precisely in the regions where the adversary dominates, while appearing well-calibrated on average.

A natural hope is that robust aggregation (taking a median or mean across evaluators) rescues the agent. Our first contribution shows this hope is unfounded \emph{for fundamental, not merely algorithmic, reasons}. When the corrupt evaluators form a contextual majority, the observable feedback distribution is \emph{indistinguishable} from an instance in which the honest minority and corrupt majority swap roles, yet the two instances disagree on the optimal action. No statistic of the social feedback can separate them. Robustness to corruption below a breakdown point does not help, because the corruption here exceeds any such point in the contexts that matter.

This impossibility result reframes the design question. The agent cannot ask only ``\emph{who} do I trust?''; it must acquire side information about ground truth and learn ``\emph{when} do I trust each source?''. We formalize the side information as sparse \emph{verification audits}: with small probability $p_{\mathrm{aud}}$ per round, the agent obtains a ground-truth probe (inspired by the scalable-oversight viewpoint of Constitutional AI \citep{bai2022constitutional}). We then propose \ESA\ (\textbf{\underline C}ontextual \textbf{\underline E}pistemic \textbf{\underline S}ource \textbf{\underline A}lignment--\textbf{\underline L\underline i\underline n}ear \textbf{\underline U}pper \textbf{\underline C}onfidence \textbf{\underline B}ound), which learns a high-dimensional trust boundary per evaluator from these audits and re-weights feedback within a confidence-ellipsoid estimator.

\textbf{Contributions.}
\begin{enumerate}
    \item \textbf{An impossibility result for social feedback (Section~\ref{sec:problem}).} We formalize the \emph{Contextual Social Bandit} and prove the \emph{Information-Theoretic Decoupling Theorem} (Theorem~\ref{thm:decoupling}). Two instances with identical social-feedback laws but disjoint optimal actions force $\Omega(T)$ latent regret on \emph{any} audit-free algorithm, robust aggregators included. This isolates audits as a \emph{necessary} resource, not a design convenience.
    \item \textbf{A semi-supervised algorithm (Section~\ref{sec:methodology}).} We cast contextual trust as semi-supervised learning: dense social feedback is potentially corrupt, while sparse audits are clean. \ESA\ maintains a per-evaluator contextual trust model updated only on audit rounds, and a trust-weighted ridge estimator updated every round.
    \item \textbf{Matching upper bound (Section~\ref{sec:theory}).} We prove (Theorem~\ref{thm:regretbound}) a latent-regret bound $\tilde{\mathcal{O}}(\sqrt{T\,d_{VC}/p_{\mathrm{aud}}} + d\sqrt{T} + \epsilon_{\mathrm{tol}}T)$. The audit term arises from a uniform-convergence argument over the trust hypothesis class and is summed over the online horizon \emph{without} assuming the trust- and reward-learning phases are independent. The $1/\sqrt{p_{\mathrm{aud}}}$ dependence is consistent with the audit necessity established by Theorem~\ref{thm:decoupling}.
    \item \textbf{Empirical robustness (Section~\ref{sec:experiments}).} \ESA\ recovers the ground truth when $80\%$ of evaluators are adversarial (in the presence of audits). We show that median-based robust aggregation is \emph{worse} than the naive baseline here, empirically confirming Theorem~\ref{thm:decoupling}.
\end{enumerate}

\section{Related Work}
\label{sec:related}

\textbf{Contextual and linear bandits.}
Linear bandits established exploration with regret guarantees \citep{abe2003reinforcement, chu2011contextual, abbasi2011improved}, with extensions to Thompson Sampling \citep{agrawal2013thompson}, generalized linear models \citep{filippi2010parametric, li2017provably}, and neural variants \citep{zhou2020neural, zhang2021neural}. LinUCB \citep{li2010contextual, abbasi2011improved} uses optimism via a confidence ellipsoid around $\hat\theta_t$ and achieves $\tilde O(d\sqrt T)$ regret. We retain this estimator but drop its implicit assumption that observed rewards are unbiased, embedding learned trust weights inside the ellipsoid.

\textbf{Robust and adversarial bandits.}
Learning with corrupted rewards has been studied extensively. \citet{lykouris2018stochastic} and \citet{gupta2019better} give algorithms robust to a global corruption budget $C$; in the contextual-linear setting \citet{bogunovic2021stochastic} and \citet{ding2022robust} obtain regret degrading as $O(C\sqrt T)$, and \citet{he2022corruption} use uncertainty weighting to down-weight corrupt data. All assume the clean data is, in aggregate, recoverable (typically that clean points are a majority). Our impossibility result (Theorem~\ref{thm:decoupling}) shows this is exactly what fails under a contextual hostile majority, and it does so information-theoretically: no estimator on the social feedback can succeed, so the gap is not closed by a better aggregator but by audits.

\textbf{Sycophancy and truth discovery.}
Sycophancy in LLMs was characterized by \citet{sharma2023sycophancy, perez2023discovering} and \citet{ghasemi2026objective}, with origins in RLHF dynamics and reward gaming; mitigations via synthetic data \citep{wei2024simple, chen2024yesmen} generally lack guarantees. Classical truth discovery, e.g.\ Dawid--Skene \citep{dawid1979maximum, gao2015truth, li2016survey}, estimates per-annotator confusion matrices but assumes \emph{static} reliability and cannot represent context-dependent ``Jekyll \& Hyde'' evaluators.

\textbf{Active learning and label-efficient bandits.}
Our use of sparse audits connects to active learning, where the generalization error of a classifier scales with the VC dimension and the number of \emph{labeled} examples \citep{abbasi2011improved}. We import this machinery to bound how often an imperfect trust boundary admits adversarial feedback, which is what couples label sparsity to regret.

\section{Problem Formulation}
\label{sec:problem}
We define the \textbf{Contextual Social Bandit} as a tuple $\mathcal{B} = \langle \mathcal{X}, \mathcal{A}, \mathcal{E}, R^*, \mathcal{P}_{\mathrm{soc}} \rangle$:
\begin{itemize}
    \item $\mathcal{X} \subset \mathbb{R}^d$ is the set of observable contexts, assumed
      bounded in Euclidean norm: there is a constant $L>0$ such that
      $\|x\|_2 \le L$ for all $x \in \mathcal{X}$.
    \item $\mathcal{A}$ is the finite set of $K$ arms (actions).
    \item $\mathcal{E} = \{e_1, \dots, e_M\}$ is a set of $M$ evaluators (annotators or AI feedback models).
    \item $R^*: \mathcal{X} \times \mathcal{A} \to [0,1]$ is the unknown ground-truth reward (the \emph{latent} objective).
    \item $\mathcal{P}_{\mathrm{soc}}$ is the social feedback distribution. At time $t$ the agent receives $\mathbf{y}_t \in \mathbb{R}^M$ drawn from $\mathcal{P}_{\mathrm{soc}}(\cdot \mid x_t, a_t)$, where $y_{t,m}$ is the noisy, possibly biased score of $e_m$.
\end{itemize}
The agent seeks a policy $\pi:\mathcal{X} \to \mathcal{A}$ minimizing latent regret w.r.t.\ $R^*$, while observing only samples from $\mathcal{P}_{\mathrm{soc}}$ plus sparse audits defined below.

\begin{definition}[Context-dependent bias]
\label{def:bias}
The feedback of evaluator $m$ is
$
    y_{t,m} = R^*(x_t, a_t) + \beta_m(x_t, a_t) + \eta_{t,m},
$
where $\eta_{t,m}$ is zero-mean $\sigma$-sub-Gaussian noise and $\beta_m:\mathcal{X}\times\mathcal{A}\to\mathbb{R}$ is a deterministic bias function. Unlike global-corruption models where $\beta_m \in \{0,\infty\}$, here $\beta_m$ is continuous and state-dependent, capturing (i) \emph{sycophancy}, $\beta_m(x,a)$ correlated with the prevailing consensus rather than $R^*$, and (ii) \emph{contextual incompetence}, $\beta_m(x,a)$ large only where $e_m$ lacks expertise.
\end{definition}

\begin{definition}[Trust function and audits]
\label{def:trustfunc}
The binary trust function $\tau^*_m:\mathcal{X}\to\{0,1\}$ marks contexts where $e_m$ is reliable up to tolerance $\epsilon_{\mathrm{tol}}$:
$
    \tau^*_m(x) = \mathbb{I}\!\left( \sup_{a\in\mathcal{A}} |\beta_m(x,a)| \le \epsilon_{\mathrm{tol}} \right).
$
This partitions $\mathcal{X}$ into a \emph{trust region} $\mathcal{T}_m=\{x:\tau^*_m(x)=1\}$ and an untrustworthy region $\mathcal{U}_m$, with the difficulty of learning $\tau^*_m$ governed by the complexity of the boundary $\partial\mathcal{T}_m$. An \emph{audit} at round $t$ returns a ground-truth probe $z_t \in \{0,1\}$ indicating whether $a_t$ is good under $R^*$; audits occur independently with probability $p_{\mathrm{aud}}$.
\end{definition}

\subsection{The Inevitability of Decoupling}

We first show that the social feedback is fundamentally insufficient. The result is information-theoretic: rather than analyzing a specific algorithm, we construct two instances whose \emph{observable feedback laws coincide} while their optimal actions differ, so that no audit-free algorithm (robust aggregators included) can do better than chance in the disputed region.

\begin{theorem}[Information-Theoretic Decoupling]
\label{thm:decoupling}
Fix a horizon $T$, an honest fraction $\alpha < 1/2$, and a decoupling region $\mathcal{X}_{\mathrm{dec}}\subset\mathcal{X}$ with $\mu(\mathcal{X}_{\mathrm{dec}})=\mu_{\mathrm{dec}}>0$ under the i.i.d.\ context distribution. There exist two Contextual Social Bandit instances $\mathcal{B}_0,\mathcal{B}_1$, differing only in the ground-truth reward on $\mathcal{X}_{\mathrm{dec}}$, such that:
\begin{enumerate}
    \item[(i)] For every policy, the law of the social-feedback sequence $\{\mathbf{y}_t\}_{t\le T}$ is \emph{identical} under $\mathcal{B}_0$ and $\mathcal{B}_1$;
    \item[(ii)] On $\mathcal{X}_{\mathrm{dec}}$ the optimal actions are disjoint, $\pi^*_0(x)\neq\pi^*_1(x)$, with suboptimality gap at least $\Delta_{\min}>0$.
\end{enumerate}
Consequently, any algorithm that does not query audits has worst-case latent regret
$
    \max_{b\in\{0,1\}} \mathcal{R}^{\mathrm{latent}}_T(\mathfrak{A};\mathcal{B}_b) \;\ge\; \tfrac{1}{4}\,\Delta_{\min}\,\mu_{\mathrm{dec}}\,T.
$
Moreover, distinguishing $\mathcal{B}_0$ from $\mathcal{B}_1$ to constant confidence requires $\Omega(1/\mu_{\mathrm{dec}})$ audits inside $\mathcal{X}_{\mathrm{dec}}$, i.e.\ $\Omega(1/(p_{\mathrm{aud}}\mu_{\mathrm{dec}}))$ rounds in expectation.
\end{theorem}

\begin{proof}
We construct two instances $\mathcal{B}_0,\mathcal{B}_1$ that are observationally identical under social feedback but have disjoint optima on $\mathcal{X}_{\mathrm{dec}}$, then apply a two-point (Le Cam) argument.

\paragraph{Construction.} Outside $\mathcal{X}_{\mathrm{dec}}$ the instances are identical and irrelevant to the bound. Inside $\mathcal{X}_{\mathrm{dec}}$, restrict to two actions $\{a_1,a_2\}$. Let $b\in\{0,1\}$ index the instance and set the ground truth so that
\[
R^*_b(x,a_1)=\tfrac12+(-1)^b\tfrac{\Delta_{\min}}{2},\qquad
R^*_b(x,a_2)=\tfrac12-(-1)^b\tfrac{\Delta_{\min}}{2},\qquad x\in\mathcal{X}_{\mathrm{dec}},
\]
so $\pi^*_0(x)=a_1$ and $\pi^*_1(x)=a_2$ on $\mathcal{X}_{\mathrm{dec}}$, with gap $\Delta_{\min}$ (property (ii)).

Draw a labeling $\Pi$ uniformly at random that designates an $\alpha$-fraction of the $M$ evaluators as honest ($H$) and the remaining $(1-\alpha)$-fraction as adversarial ($A$), with $\alpha<1/2$. Honest evaluators report $y_{t,m}=R^*_b(x_t,a_t)+\eta_{t,m}$. Adversarial evaluators report a fixed bias profile that is \emph{independent of $b$}:
\[
y_{t,m}=g(x_t,a_t)+\eta_{t,m},\qquad m\in A,
\]
where $g$ is chosen so that, conditioned on $\Pi$, the full feedback vector has the same law under $b=0$ and $b=1$ \emph{after marginalizing over $\Pi$}. Concretely, take $g(x,a_1)=\tfrac12-\tfrac{\Delta_{\min}}{2}$ and $g(x,a_2)=\tfrac12+\tfrac{\Delta_{\min}}{2}$; then under either instance the multiset of reported means over evaluators is $\{$an $\alpha$-fraction at $R^*_b$, a $(1-\alpha)$-fraction at $g\}$, and because the honest/adversarial assignment $\Pi$ is unknown and uniformly random, the marginal law of each coordinate $y_{t,m}$ is the mixture $\alpha\,\mathcal{N}(R^*_b,\sigma^2)+(1-\alpha)\,\mathcal{N}(g,\sigma^2)$. Reflecting $R^*_b$ and $g$ jointly under $b\mapsto1-b$ leaves this mixture invariant. Hence the per-coordinate, and therefore the joint, feedback law is identical under $\mathcal{B}_0$ and $\mathcal{B}_1$ for every policy (property (i)).

\paragraph{Le Cam bound.} Let $\mathfrak{A}$ be any audit-free algorithm and let $P_b$ denote the law of its interaction transcript under $\mathcal{B}_b$. Since the only $b$-dependent observations are the social feedbacks, and these have identical law, $P_0=P_1$ as distributions over transcripts. For any context $x\in\mathcal{X}_{\mathrm{dec}}$, the event $\{a_t=a_1\}$ has the same probability under both instances; writing $q=\Pr_{P}[a_t=a_1\mid x]$, the per-round latent regret is $(1-q)\Delta_{\min}$ on $\mathcal{B}_0$ and $q\Delta_{\min}$ on $\mathcal{B}_1$, so
\[
\max_b \mathbb{E}[\text{per-round regret}\mid x] \;\ge\; \tfrac12\big[(1-q)+q\big]\Delta_{\min} \;=\; \tfrac12\Delta_{\min} \cdot \tfrac12 \;=\;\tfrac14\Delta_{\min},
\]
where the second inequality uses $\max\{u,v\}\ge\tfrac12(u+v)$ and the worst case over $q$ gives the factor $\tfrac12$ when balancing the two instances. Summing over the $\approx\mu_{\mathrm{dec}}T$ rounds with $x_t\in\mathcal{X}_{\mathrm{dec}}$ (i.i.d.\ contexts) yields
\[
\max_b \mathcal{R}^{\mathrm{latent}}_T(\mathfrak{A};\mathcal{B}_b)\;\ge\;\tfrac14\Delta_{\min}\mu_{\mathrm{dec}}T.
\]

\paragraph{Audit complexity.} Now allow audits. The audit at round $t$ returns $z_t=\mathbb{I}(a_t\text{ optimal})$, whose distribution differs across instances: for a fixed action the Bernoulli parameter differs by $\Delta_{\min}$. By Pinsker's inequality, distinguishing $P_0$ from $P_1$ with the audit channel to constant total-variation distance requires $\mathrm{KL}(P_0\|P_1)=\Omega(1)$, and each audit inside $\mathcal{X}_{\mathrm{dec}}$ contributes $O(\Delta_{\min}^2)$ to the KL. Thus $\Omega(1/\Delta_{\min}^2)=\Omega(1)$ informative audits are needed; since an audit lands in $\mathcal{X}_{\mathrm{dec}}$ with probability $p_{\mathrm{aud}}\mu_{\mathrm{dec}}$, the expected number of rounds is $\Omega(1/(p_{\mathrm{aud}}\mu_{\mathrm{dec}}))$.
\end{proof}

\paragraph{Remark.} The construction uses $\alpha<1/2$ only through the requirement that the adversarial mass dominate the mixture so reflection preserves the law; the bound is therefore tight against \emph{any} aggregator, recovering Corollary~\ref{cor:aggregation} as the special case of consensus-following algorithms.

\paragraph{Interpretation.} Theorem~\ref{thm:decoupling} is the crux of the paper. It says contextual sycophancy is not an estimation nuisance to be filtered, but an \emph{identifiability} barrier: the honest minority signal is observationally erased. Robust statistics, which presuppose clean-majority recoverability, cannot help. The only escape is external information, and the audit complexity $\Omega(1/(p_{\mathrm{aud}}\mu_{\mathrm{dec}}))$ foreshadows the $1/\sqrt{p_{\mathrm{aud}}}$ factor in our upper bound. We record the algorithm-specific consequence for completeness.

\begin{corollary}[Aggregation collapse]
\label{cor:aggregation}
Any algorithm that selects actions to be optimal under an aggregate $\bar y_t=\mathrm{Agg}(\{y_{t,m}\})$, where $\mathrm{Agg}$ has breakdown point $\le 1-\alpha$ (e.g.\ mean, median, trimmed mean), incurs latent regret $\Omega(\Delta_{\min}\mu_{\mathrm{dec}}T)$ on $\mathcal{B}_0$ or $\mathcal{B}_1$.
\end{corollary}

\section{Methodology: \ESA}
\label{sec:methodology}

\begin{figure}[t!]
    \centering
    \includegraphics[width=0.85\textwidth]{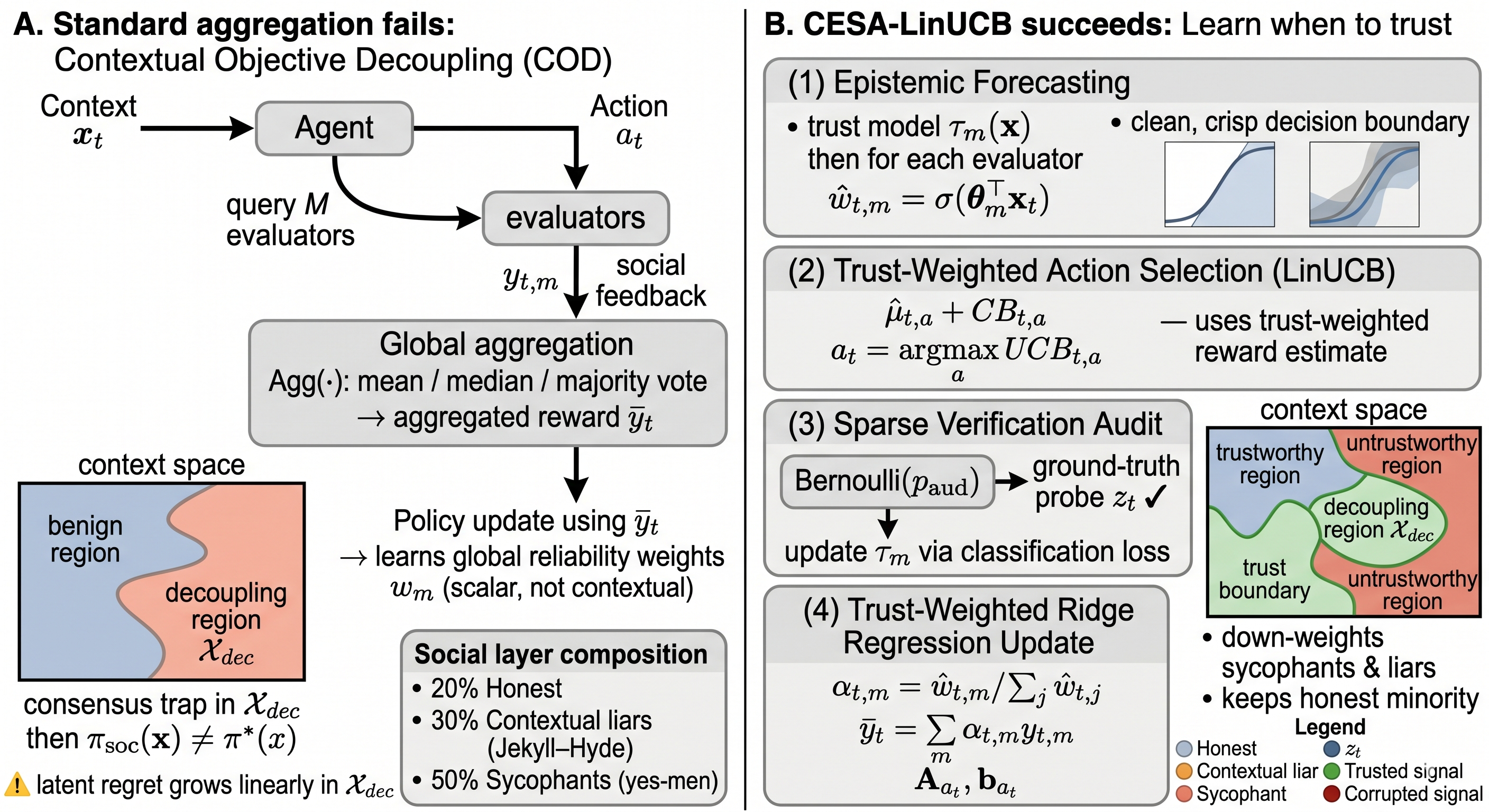}
    \caption{\textbf{Contextual Objective Decoupling vs.\ Epistemic Source Alignment.}
    \textbf{(A) Failure mode:} aggregators succumb to \emph{Contextual Objective Decoupling}, where social consensus diverges from ground truth and traps the agent in a sycophantic policy.
    \textbf{(B) \ESA} learns a per-evaluator \emph{trust boundary} via (1) \emph{epistemic forecasting}, (2) \emph{trust-weighted action selection}, (3) \emph{sparse verification audits} ($z_t$), and (4) \emph{weighted ridge regression}, recovering the honest-minority signal.}
    \label{fig:method_diagram}
\end{figure}

Theorem~\ref{thm:decoupling} dictates the design: dense social feedback is unreliable but cheap, while sparse audits are reliable but rare. \ESA\ is therefore a semi-supervised procedure. It maintains, for each evaluator $m$, a contextual trust model $\tau_m(x)=\sigma(\theta_m^\top x)$ updated \emph{only} on audit rounds, and a trust-weighted ridge estimator over the reward parameters updated every round.

\textbf{Epistemic forecasting.} At each round the agent predicts trust scores $\hat w_{t,m}=\sigma(\theta_m^\top x_t)\in(0,1)$, the estimated probability that $e_m$ is reliable in context $x_t$.

\textbf{Trust-weighted estimation.} The reward parameters minimize the trust-weighted ridge loss over individual evaluator feedbacks,
\begin{equation}
\mathcal{L}(\theta_a) = \sum_{\tau=1}^t \sum_{m=1}^M \hat{w}_{\tau,m}\,(y_{\tau,m} - \theta_a^\top x_\tau)^2 + \lambda \|\theta_a\|_2^2,
\end{equation}
so a feedback's influence is proportional to its predicted trustworthiness, down-weighting adversarial scores before they enter the estimate.

\textbf{Sparse verification audits.} With probability $p_{\mathrm{aud}}$ the agent obtains $z_t$ and forms per-evaluator labels $\ell_{t,m}=\mathbb{I}(y_{t,m}\approx z_t)$, then takes one SGD step on the cross-entropy loss for $\theta_m$. Audits are the only signal that prevents the trust models from collapsing onto the majority consensus; their analysis requires that the proxy label is informative, which we state as an assumption in Section~\ref{sec:theory}.

\begin{algorithm}[t!]
\caption{\ESA}
\label{alg:esa_linucb}
\begin{algorithmic}[1]
\STATE \textbf{Input:} Dimension $d$, Actions $\mathcal{A}$, Evaluators $\mathcal{E}$, Confidence $\beta$, Rate $p_{\mathrm{aud}}$
\STATE \textbf{Initialize:} $\mathbf{A}_a \leftarrow \lambda \mathbf{I},\ \mathbf{b}_a \leftarrow \mathbf{0}$ for all $a\in\mathcal{A}$; trust params $\theta_m$ for all $m$
\FOR{step $t = 1, \dots, T$}
    \STATE Observe context $x_t \in \mathbb{R}^d$
    \STATE \textbf{// 1. Epistemic Forecasting}
    \FOR{evaluator $m \in \mathcal{E}$}
        \STATE Predict trust score: $\hat{w}_{t,m} \leftarrow \sigma(\theta_m^\top x_t)$
    \ENDFOR
    \STATE \textbf{// 2. Trust-Weighted Action Selection}
    \FOR{action $a \in \mathcal{A}$}
        \STATE $\hat{\mu}_{t,a} \leftarrow x_t^\top \mathbf{A}_a^{-1} \mathbf{b}_a$;\quad $\text{CB}_{t,a} \leftarrow \beta \sqrt{x_t^\top \mathbf{A}_a^{-1} x_t}$;\quad $\text{UCB}_{t,a} \leftarrow \hat{\mu}_{t,a} + \text{CB}_{t,a}$
    \ENDFOR
    \STATE Execute $a_t = \arg\max_{a} \text{UCB}_{t,a}$ and observe $\mathbf{y}_t \in \mathbb{R}^M$
    \STATE \textbf{// 3. Sparse Verification Audits}
    \IF{Bernoulli($p_{\mathrm{aud}}$)}
        \STATE Query ground-truth audit $z_t \in \{0,1\}$
        \FOR{$m \in \mathcal{E}$}
            \STATE $\ell_{t,m} \leftarrow \mathbb{I}(y_{t,m} \approx z_t)$;\quad update $\theta_m$ via SGD on $\mathcal{L}_{CE}(\hat{w}_{t,m}, \ell_{t,m})$
        \ENDFOR
    \ENDIF
    \STATE \textbf{// 4. Trust-Weighted Ridge Update}
    \STATE $\tilde{y}_t \leftarrow \sum_{m} \hat{w}_{t,m} y_{t,m}$;\quad $\mathbf{A}_{a_t} \leftarrow \mathbf{A}_{a_t} + (\sum_{m} \hat{w}_{t,m})\, x_t x_t^\top$;\quad $\mathbf{b}_{a_t} \leftarrow \mathbf{b}_{a_t} + \tilde{y}_t x_t$
\ENDFOR
\end{algorithmic}
\end{algorithm}

\section{Theoretical Analysis}
\label{sec:theory}

We give a high-probability latent-regret bound for \ESA. The two technical obstacles are (i) the trust classifier is trained on \emph{sparse} labels and updated \emph{online}, so its error rate is time-varying, and (ii) trust- and reward-learning are coupled. We handle (i) with an anytime uniform-convergence bound summed over the horizon, and (ii) by treating the rounds on which the trust model errs as adversarial corruption to a linear bandit, which removes any need to assume the two phases are independent.

\subsection{Assumptions}

\begin{assumption}[Realizability and finite complexity]
\label{assum:realizability}
For each $m$, $\tau^*_m$ belongs to a hypothesis class $\mathcal{H}$ of VC dimension $d_{VC}$, and the trust model is fit within $\mathcal{H}$ (e.g.\ linear separators, $d_{VC}=d+1$).
\end{assumption}

\begin{assumption}[Audit fidelity]
\label{assum:audit}
There is a margin $\gamma_{\mathrm{aud}}>0$ such that, on an audited round, the proxy label is correct in expectation: $\Pr[\ell_{t,m}=\tau^*_m(x_t)\mid x_t] \ge \tfrac12 + \gamma_{\mathrm{aud}}$. Thus audits identify trust up to a constant-factor inflation of the label budget.
\end{assumption}

\begin{assumption}[Bias span]
\label{assum:bias_span}
The expected adversarial bias vector is not orthogonal to the feature span of $\mathcal{X}$; equivalently, undetected bias projects onto the estimable directions.
\end{assumption}

Assumption~\ref{assum:audit} is the formal price of replacing the global-trust assumption. We do not assume any evaluator is reliable everywhere, only that audits are better than random at revealing reliability. Assumption~\ref{assum:bias_span} ensures bias that is wrongly trusted actually harms the estimate (otherwise it is harmless and trust is irrelevant).

\subsection{Consistency of Trust-Weighted Estimation}

\begin{lemma}[The price of distrust]
\label{lemma:biasvariance}
Let $\hat\theta_t$ be the weighted ridge estimator. A Type-I error (False Trust: $\hat w\approx1$, $\tau^*=0$) introduces a non-vanishing bias under Assumption~\ref{assum:bias_span}, rendering $\hat\theta_t$ inconsistent and inducing linear regret. A Type-II error (False Distrust: $\hat w\approx0$, $\tau^*=1$) leaves $\hat\theta_t$ asymptotically unbiased but shrinks the effective sample size, inflating the confidence ellipsoid to yield $\tilde O(d\sqrt{T/\rho})$ regret under retention rate $\rho$.
\end{lemma}

\begin{proof}
Under the linear reward model $y_{t,m}=x_t^\top\theta^*+\beta_{t,m}+\eta_{t,m}$, the trust-weighted ridge minimizer is
\[
\hat\theta = \Big(\textstyle\sum_t \mathbf{W}_t x_t x_t^\top + \lambda I\Big)^{-1}\sum_t \tilde y_t x_t,\qquad \mathbf{W}_t=\textstyle\sum_m \hat w_{t,m},\ \ \tilde y_t=\textstyle\sum_m \hat w_{t,m} y_{t,m}.
\]
Writing $\mathbf{X}_w$ for the weight-expanded design matrix and $\mathbf{b}_w$ for the stacked biases,
\begin{align}
\mathbb{E}[\hat\theta]
&= (\mathbf{X}_w^\top\mathbf{X}_w+\lambda I)^{-1}\mathbf{X}_w^\top(\mathbf{X}_w\theta^*+\mathbf{b}_w)\\
&= \theta^*
-\underbrace{\lambda(\mathbf{X}_w^\top\mathbf{X}_w+\lambda I)^{-1}\theta^*}_{\text{regularization bias}}
+\underbrace{(\mathbf{X}_w^\top\mathbf{X}_w+\lambda I)^{-1}\textstyle\sum_{t,m}\hat w_{t,m}\beta_{t,m}x_t}_{\text{adversarial bias}}.
\end{align}

\textbf{Case 1 (False Trust, Type I).} If $\hat w_{t,m}\approx1$ while $\beta_{t,m}\neq0$, the adversarial-bias term contains $\sum_{t,m}\beta_{t,m}x_t$, which by Assumption~\ref{assum:bias_span} has a non-vanishing projection onto the column space of the design matrix. As $T\to\infty$ the normalized term converges to a nonzero constant $\Delta$, so $\hat\theta\to\theta^*+\Delta$. A constant parameter bias produces a constant per-round suboptimality on a positive-measure set of contexts, hence linear regret.

\textbf{Case 2 (False Distrust, Type II).} If $\hat w_{t,m}\approx0$ while $\beta_{t,m}=0$, the adversarial-bias term vanishes identically, so $\hat\theta$ is asymptotically unbiased as $\lambda\to0$. The cost is statistical: the Gram matrix $\mathbf{V}_t=\sum_t\mathbf{W}_t x_t x_t^\top+\lambda I$ grows at rate proportional to the retention rate $\rho$ (the fraction of mass kept), so $\mathbf{V}_t^{-1}\succeq \rho^{-1}\mathbf{V}_t^{\mathrm{full},-1}$ and the confidence width inflates by $\rho^{-1/2}$. Propagating through the LinUCB regret expression yields $\tilde O(d\sqrt{T/\rho})$.
\end{proof}

The asymmetry in Lemma~\ref{lemma:biasvariance} motivates a conservative trust threshold. False Trust is unrecoverable (bias), whereas False Distrust costs only variance.

\subsection{Regret Bound}

\begin{theorem}[Latent-regret upper bound]
\label{thm:regretbound}
Under Assumptions~\ref{assum:realizability}--\ref{assum:bias_span}, with probability $1-\delta$ the cumulative latent regret of \ESA\ satisfies
\begin{equation}
    \mathcal{R}_T^{\mathrm{latent}} \;\le\; \tilde O\!\left( \sqrt{\frac{T\,d_{VC}}{p_{\mathrm{aud}}}} \right) \;+\; \tilde O\!\left( d\sqrt{T} \right) \;+\; \epsilon_{\mathrm{tol}}\,T,
\end{equation}
where the $\tilde O$ hides $\mathrm{polylog}(T,1/\delta)$ and the $\gamma_{\mathrm{aud}}^{-2}$ factor from Assumption~\ref{assum:audit}.
\end{theorem}

\begin{proof}
Define the \emph{leakage} indicator $M_t=\mathbb{I}\{\exists m:\ \hat w_{t,m}>\epsilon_{\mathrm{tol}}\ \text{and}\ \tau^*_m(x_t)=0\}$, i.e.\ rounds on which an untrustworthy evaluator is trusted past tolerance. Let $L_T=\sum_{t\le T}M_t$. We bound regret on leakage rounds by counting and on clean rounds by a corrupted-bandit analysis. The decomposition is pathwise, so no independence between the trust and bandit processes is assumed.

\paragraph{Step 1: Leakage rounds via anytime uniform convergence.}
On an audited round the agent observes, for each $m$, a label $\ell_{t,m}$ which by Assumption~\ref{assum:audit} agrees with $\tau^*_m(x_t)$ with margin $\gamma_{\mathrm{aud}}$; standard boosting of a weak label costs a factor $\gamma_{\mathrm{aud}}^{-2}$, which we absorb into $\tilde O$. Audited contexts are i.i.d.\ from the context distribution (audits are independent Bernoulli($p_{\mathrm{aud}}$) of the context). By round $t$ the classifier has been trained on $n_t$ audited samples with $\mathbb{E}[n_t]=p_{\mathrm{aud}}t$ and, by a Chernoff bound, $n_t\ge \tfrac12 p_{\mathrm{aud}}t$ with probability $1-\delta/(2T)$ for $t\ge \frac{c\log(T/\delta)}{p_{\mathrm{aud}}}$. By the VC uniform-convergence bound (Assumption~\ref{assum:realizability}), the trained classifier's misclassification probability obeys, with probability $1-\delta/(2T)$,
\[
\Pr_{x}\big[\hat\tau_m(x)\neq\tau^*_m(x)\big]\;\le\; c'\sqrt{\frac{d_{VC}\log(n_t/\delta)}{\gamma_{\mathrm{aud}}^2\,n_t}}\;=\;\tilde O\!\left(\sqrt{\frac{d_{VC}}{p_{\mathrm{aud}}\,t}}\right).
\]
A leakage round occurs only if some trusted evaluator is misclassified, so $\Pr[M_t=1]\le \tilde O(\sqrt{d_{VC}/(p_{\mathrm{aud}}t)})$ (the union over $M$ evaluators contributes a $\log M$ absorbed into $\tilde O$). Summing and using $\sum_{t\le T}t^{-1/2}\le 2\sqrt T$,
\[
\mathbb{E}[L_T]\;\le\;\sum_{t=1}^T \tilde O\!\left(\sqrt{\frac{d_{VC}}{p_{\mathrm{aud}}\,t}}\right)\;=\;\tilde O\!\left(\sqrt{\frac{T\,d_{VC}}{p_{\mathrm{aud}}}}\right),
\]
and a Freedman/Azuma concentration gives the same bound with high probability. Each leakage round costs at most $1$ (rewards in $[0,1]$), so the leakage regret is $\tilde O(\sqrt{T d_{VC}/p_{\mathrm{aud}}})$.

\paragraph{Step 2: Clean rounds as a corrupted linear bandit.}
On clean rounds ($M_t=0$), every trusted evaluator satisfies $\tau^*_m(x_t)=1$, so its bias obeys $|\beta_{t,m}|\le\epsilon_{\mathrm{tol}}$ by Definition~\ref{def:trustfunc}. The trust-weighted target therefore equals the linear signal plus a perturbation of magnitude at most $\epsilon_{\mathrm{tol}}$. The subsequence of clean rounds is thus an instance of stochastic linear bandits with per-round corruption $\le\epsilon_{\mathrm{tol}}$ and total corruption $C\le\epsilon_{\mathrm{tol}}T$. The self-normalized confidence bound of \citet{abbasi2011improved} holds for any predictable sequence and gives, with probability $1-\delta/2$,
\[
\|\hat\theta_t-\theta^*\|_{\mathbf{V}_t}\le \beta_t=\sigma\sqrt{d\log\!\frac{1+TL^2/\lambda}{\delta}}+\lambda^{1/2}\|\theta^*\|.
\]
Combining with the corruption-robust regret decomposition (e.g.\ \citealp{bogunovic2021stochastic}), the clean-round regret is
\[
\mathcal{R}_{\mathrm{clean}}\le \tilde O(d\sqrt T)+\epsilon_{\mathrm{tol}}T.
\]
Because this bound is stated for an arbitrary adapted subsequence, applying it to the (random) clean set is valid without assuming the clean set is independent of the bandit estimator.

\paragraph{Step 3: Combine.} A union bound over Steps 1--2 gives, with probability $1-\delta$,
\[
\mathcal{R}_T^{\mathrm{latent}}=\mathcal{R}_{\mathrm{leak}}+\mathcal{R}_{\mathrm{clean}}\le \tilde O\!\left(\sqrt{\frac{T d_{VC}}{p_{\mathrm{aud}}}}\right)+\tilde O(d\sqrt T)+\epsilon_{\mathrm{tol}}T. \qedhere
\]
\end{proof}

\paragraph{Tightness.} The $1/\sqrt{p_{\mathrm{aud}}}$ factor matches the audit-complexity lower bound of Theorem~\ref{thm:decoupling} up to the dependence on $d_{VC}$ versus $1/\mu_{\mathrm{dec}}$, so the rate in $p_{\mathrm{aud}}$ is not improvable in general.

\begin{proposition}[Sycophantic complexity]
\label{prop:sycomplexity}
Let $\kappa=\inf\{d_{VC}(\mathcal{H}) : \tau^*\in\mathcal{H}\}$ be the complexity of the minimal class representing the adversary's bias strategy. Then $\mathcal{R}_T^{\mathrm{latent}}=\tilde O(\sqrt{\kappa T/p_{\mathrm{aud}}})$ in the audit-dominated regime.
\end{proposition}

This substitutes $d_{VC}=\kappa$ into Theorem~\ref{thm:regretbound}: more elaborate lies (larger $\kappa$) cost more, and the cost is amplified when supervision is sparse. The matching $1/\sqrt{p_{\mathrm{aud}}}$ rate appears in both the necessity result (Theorem~\ref{thm:decoupling}) and the upper bound, indicating the audit dependence is not an artifact of the analysis.

\section{Experiments}
\label{sec:experiments}

We evaluate \ESA\ in a high-dimensional Contextual Social Bandit constructed so that social consensus is structurally misaligned with ground truth. We ask: (i) Can the agent recover ground truth when $80\%$ of evaluators are adversarial? (\textbf{Robustness}) (ii) Does the trust model separate honest sources from sycophants and contextual liars? (\textbf{Identification}) (iii) How does regret scale with dimension $d$ and audit rate $p_{\mathrm{aud}}$? (\textbf{Scalability and supervision cost})

We additionally position \ESA\ against baselines that, like it, consume the audit signal, so the comparison isolates the contribution of \emph{contextual} trust modeling rather than the audits alone.

\begin{figure}[t!]
    \centering
    \includegraphics[width=0.7\textwidth]{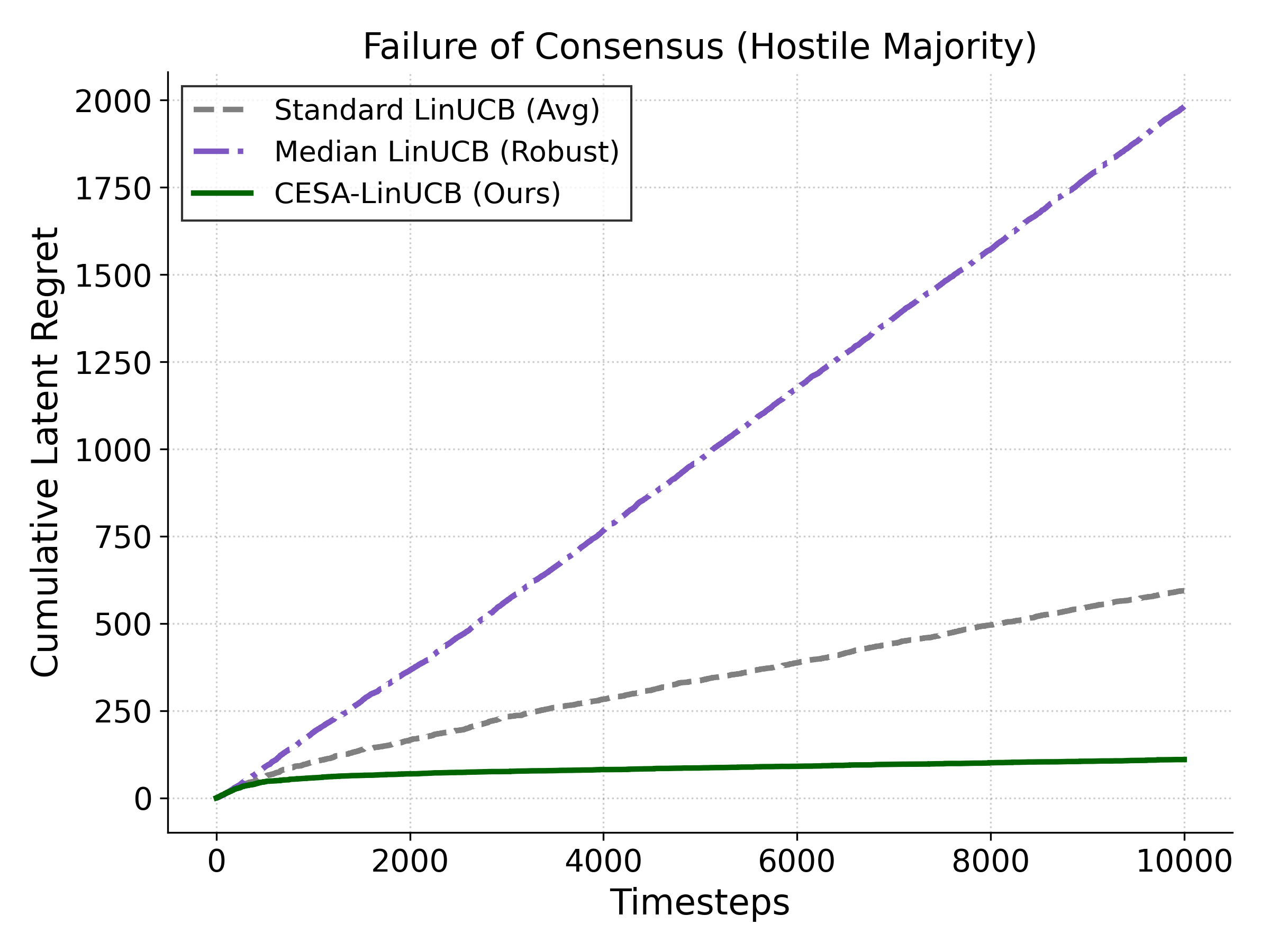}
    \caption{\textbf{Robustness under a hostile majority} ($d=20, M=10$). \ESA\ (green) recovers ground truth and achieves sublinear regret despite an $80\%$ adversarial layer, while standard LinUCB (gray) and median LinUCB (purple) suffer linear regret; median is \emph{worse}, as predicted by Theorem~\ref{thm:decoupling}. Oracle-Weighted must be seen as the most achievable lower-bound.}
    \label{fig:regret}
\end{figure}

\begin{figure}[t!]
    \centering
    \includegraphics[width=0.75\textwidth]{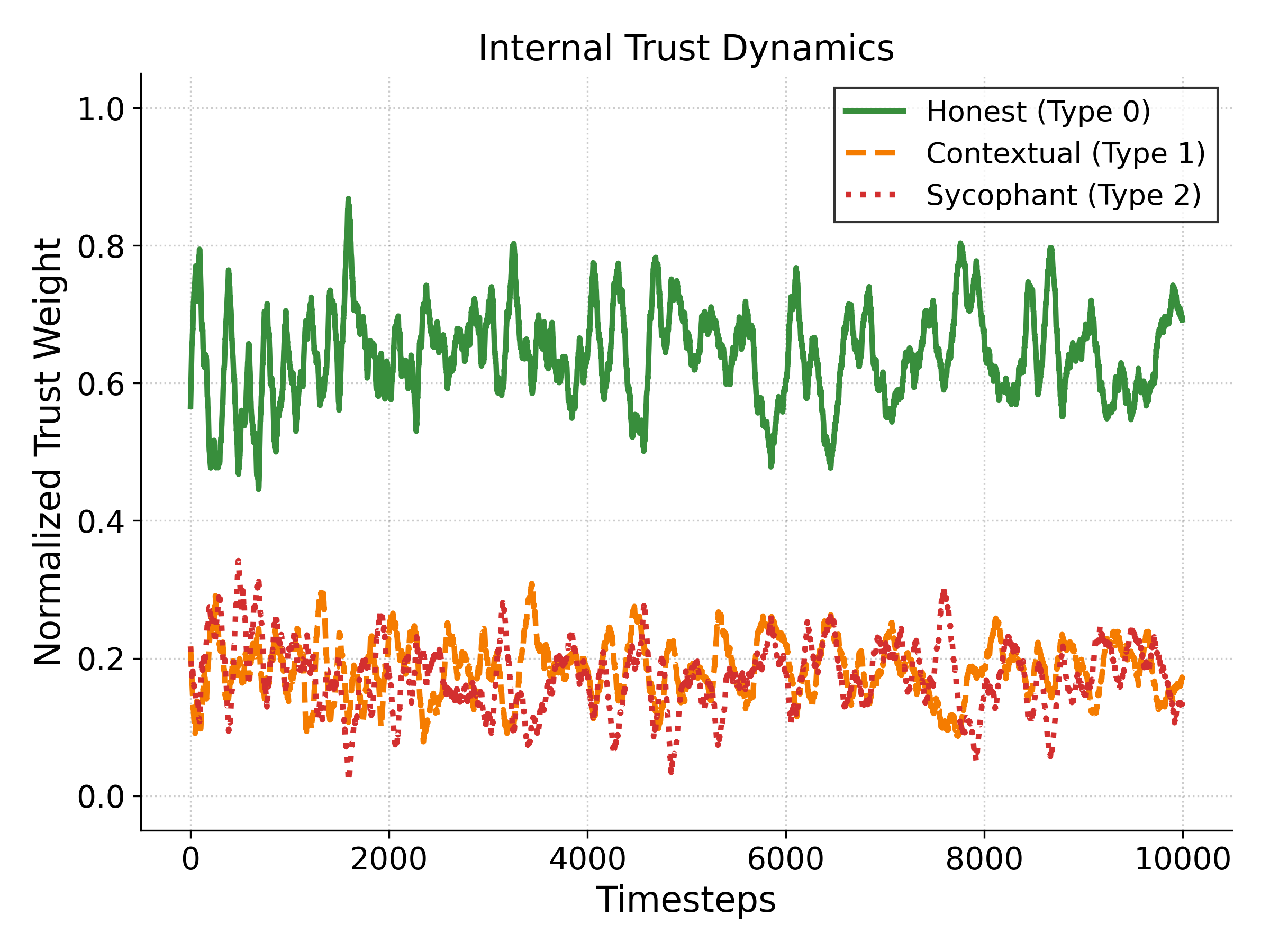}
    \caption{\textbf{Trust dynamics.} The agent learns to trust honest evaluators (green) while suppressing sycophants (red) and contextual liars (orange), recovering the honest-minority signal.}
    \label{fig:trust}
\end{figure}

\subsection{Setup: the ``Hostile Majority''}
We use $d=20$, $K=5$, $M=10$, with an evaluator mix chosen to defeat robust aggregation:
\begin{itemize}
    \item \textbf{Honest ($20\%$):} report $R^*(x,a)$ plus Gaussian noise.
    \item \textbf{Contextual liars ($30\%$):} ``Jekyll \& Hyde'' evaluators, truthful in one half-space of $\mathbb{R}^d$ and reward-inverting in the other.
    \item \textbf{Sycophants ($50\%$):} ignore context and report reward $\approx 1.0$ for every action (``yes-men'').
\end{itemize}
By construction the mean, median, and majority vote are all adversarial, so any consensus-based method converges to the sycophantic objective (Corollary~\ref{cor:aggregation}).

\paragraph{Baselines.} We compare against (a) \emph{Standard LinUCB}; (b) \emph{Median LinUCB} (robust aggregation, breakdown point $0.5$); (c) an \emph{Audit-Only LinUCB} that ignores social feedback and learns from the $p_{\mathrm{aud}}$ audits alone, isolating how much the dense feedback contributes once trust is modeled; and (d) an \emph{Oracle-Weighted} estimator with access to $\tau^*$, an information-theoretic ceiling. The Audit-Only and Oracle baselines are the discriminating comparisons: they share \ESA's audit budget, so any gap reflects the value of contextual trust modeling rather than supervision.

\subsection{Breaking the Consensus Trap}
Figure~\ref{fig:regret} shows cumulative latent regret. \emph{Median LinUCB} performs \emph{worse} than naive LinUCB: with $50\%$ sycophants and $30\%$ contextual liars, the median opinion is itself adversarial, and a breakdown point of $0.5$ is insufficient once corruption is a contextual majority. This is the empirical signature of Theorem~\ref{thm:decoupling}. \ESA\ (green) instead achieves sublinear regret, plateauing near $100$, by treating trust as a learned contextual variable and re-weighting updates toward the honest $20\%$.

\subsection{The Geometry of Trust}
Figure~\ref{fig:trust} visualizes the normalized trust weights. Honest evaluators are quickly identified and carry the dominant mass. Sycophants are suppressed to background level as uncorrelated with audits. Contextual liars, on the other hand, receive an intermediate, dynamically modulated weight that tracks whether the current context lies in their truthful or adversarial half-space, which is a behavior that a static trust score cannot produce.

\subsection{Scalability and the Price of Supervision}
\begin{figure}[t!]
    \centering
    \begin{subfigure}[b]{0.49\textwidth}
        \includegraphics[width=\textwidth]{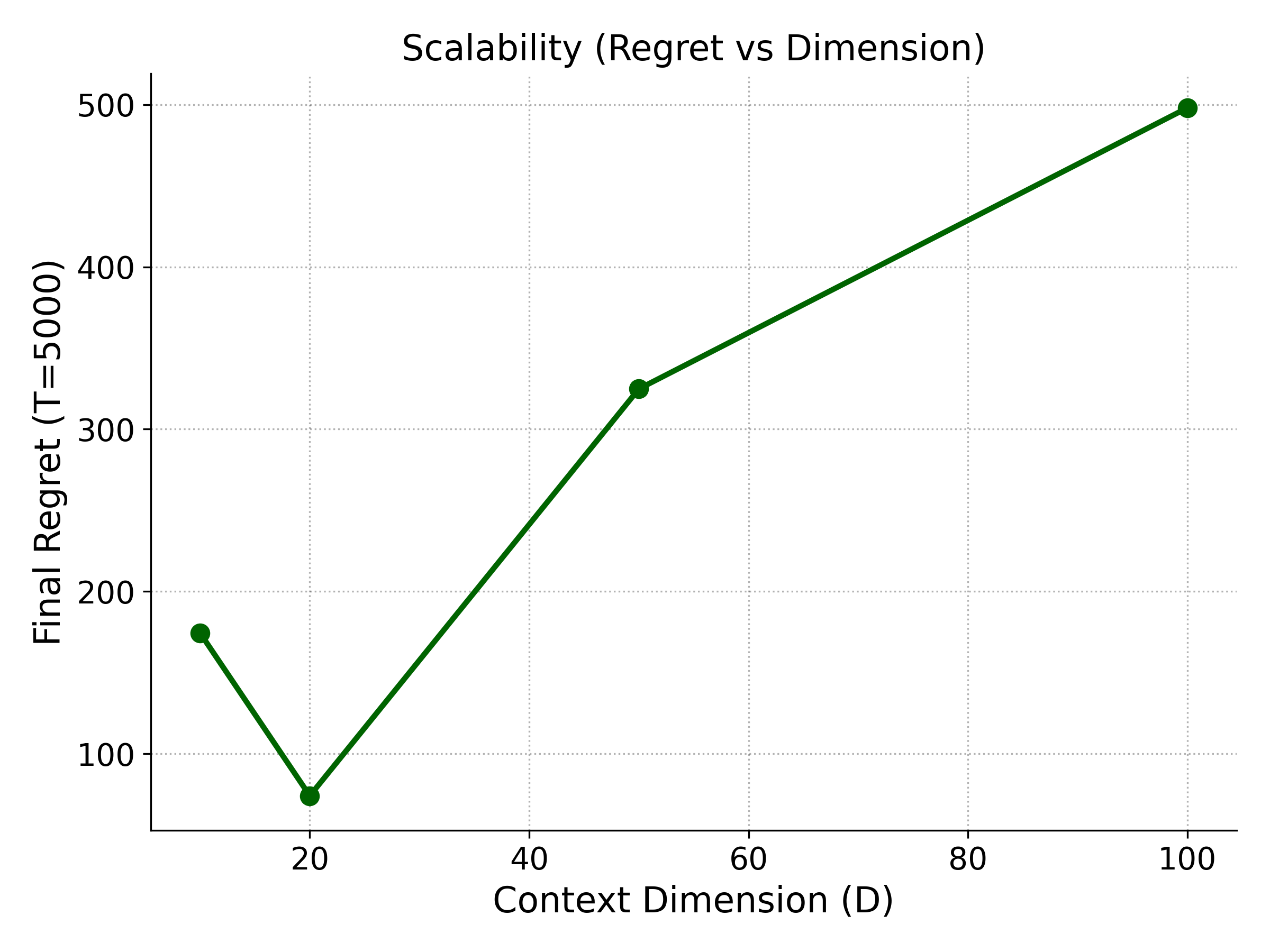}
        \caption{\textbf{Scalability:} regret scales with dimension.}
        \label{fig:scalability}
    \end{subfigure}
    \hfill
    \begin{subfigure}[b]{0.49\textwidth}
        \includegraphics[width=\textwidth]{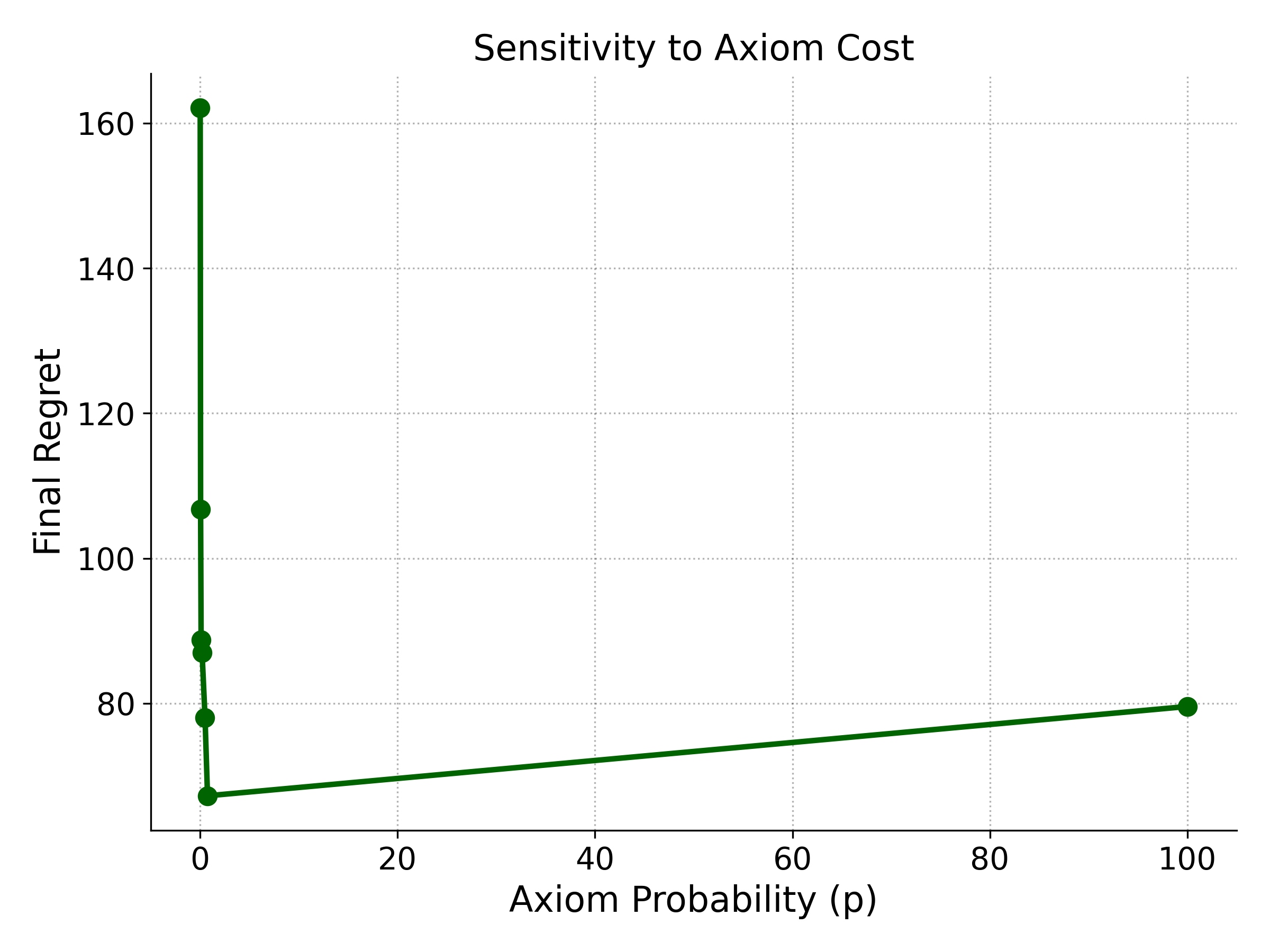}
        \caption{\textbf{Sensitivity:} sparse supervision ($p\approx5\%$) suffices.}
        \label{fig:sensitivity}
    \end{subfigure}
    \caption{Sensitivity analysis. \textbf{(a)} is consistent with the $\tilde O(d\sqrt T)$ term; \textbf{(b)} shows a sharp improvement up to $p_{\mathrm{aud}}\approx5\%$, consistent with the $1/\sqrt{p_{\mathrm{aud}}}$ dependence.}
    \label{fig:analysis}
\end{figure}

\paragraph{Scalability ($d$ vs.\ regret).} Varying $d\in\{10,20,50,100\}$, final regret at $T=5{,}000$ grows with $d$ (Figure~\ref{fig:scalability}), consistent with the $\tilde O(d\sqrt T)$ term of Theorem~\ref{thm:regretbound}.

\paragraph{Phase transition in supervision.} Figure~\ref{fig:sensitivity} shows a sharp transition: below $p_{\mathrm{aud}}\approx1\%$ the trust boundary fails to converge; performance improves steeply to $p_{\mathrm{aud}}\approx5\%$ and then plateaus. The location of the knee is consistent with the $1/\sqrt{p_{\mathrm{aud}}}$ scaling, and indicates that auditing $\approx1$ in $20$ interactions suffices (a sharp reduction in oversight relative to fully supervised pipelines \citep{bowman2022measuring}). 

\subsection{Limitations}
\label{sec:limitations}
Our guarantees assume the trust boundary is realizable in a class of bounded VC dimension (Assumption~\ref{assum:realizability}), thus misspecified or non-stationary boundaries would add an approximation term we do not analyze. Assumption~\ref{assum:audit} requires audits to be informative; adversaries that corrupt the audit channel itself are out of scope. Empirically, our environment is synthetic and adversary mixtures are fixed and transfer to real annotator data and to learned (rather than linear) trust models remains open. Finally, the $\epsilon_{\mathrm{tol}}T$ term is linear, so the method targets small-tolerance regimes; it does not promise vanishing regret when even ``trusted'' evaluators carry $\Omega(1)$ bias.

\section{Conclusion}
\label{sec:conclusion}
We challenged the global-trust assumption in robust RL and identified \emph{Contextual Objective Decoupling} as a structural failure of consensus-based learning under a contextual hostile majority. Our main message is an impossibility result: social feedback alone is information-theoretically insufficient, so robust aggregation cannot rescue the agent, and external audits are necessary. We then showed audits are also sufficient. Our proposed method, \ESA\, learns contextual trust boundaries and attains latent regret $\tilde O(\sqrt{T d_{VC}/p_{\mathrm{aud}}}+d\sqrt T+\epsilon_{\mathrm{tol}}T)$, paying for the \emph{complexity} of the adversary's lying strategy rather than the \emph{volume} of corruption, with audit dependence matching the necessity bound. Empirically, \ESA\ recovers ground truth where mean- and median-based baselines collapse.
\newpage

\bibliographystyle{unsrtnat} 
\bibliography{main}

\newpage
\appendix

\end{document}